\documentclass[runningheads]{llncs}
\usepackage{graphicx}
\usepackage{amssymb}
\usepackage{amsmath}
\usepackage{bm}
\usepackage[small]{caption}
\usepackage[english]{babel}
\usepackage{stmaryrd}
\usepackage[whole]{bxcjkjatype}
\usepackage{times}

\DeclareMathOperator*{\argmax}{arg\,max}

\newcommand{\vsim}{\mathrel{\scalebox{1}[1.5]{$\shortmid$}\mkern-3.1mu\raisebox{0.1ex}{$\sim$}}}
\newcommand{\vapprox}{\mathrel{\scalebox{1}[1.5]{$\shortmid$}\mkern-3.1mu\raisebox{0.1ex}{$\approx$}}}

\begin{document}
\title{Bayes Meets Entailment and Prediction: Commonsense Reasoning with Non-monotonicity, Paraconsistency and Predictive Accuracy\thanks{This paper is a substantial extension of the arXiv eprint \cite{kido2020bayesian}. This paper was submitted to AAAI 2021 and rejected.}}
%
\titlerunning{Bayes Meets Entailment and Prediction}
\author{Hiroyuki Kido\inst{1} \and
Keishi Okamoto\inst{2}}
\institute{Cardiff University, UK\\
\email{KidoH@cardiff.ac.uk} \and
National Institute of Technology, Sendai College, Japan\\
\email{okamoto@sendai-nct.ac.jp}}
\begin{sloppypar}
\maketitle
\begin{abstract}
The recent success of Bayesian methods in neuroscience and artificial intelligence gives rise to the hypothesis that the brain is a Bayesian machine. Since logic and learning are both practices of the human brain, it leads to another hypothesis that there is a Bayesian interpretation underlying both logical reasoning and machine learning. In this paper, we introduce a generative model of logical consequence relations. It formalises the process of how the truth value of a sentence is probabilistically generated from the probability distribution over states of the world. We show that the generative model characterises a classical consequence relation, paraconsistent consequence relation and nonmonotonic consequence relation. In particular, the generative model gives a new consequence relation that outperforms them in reasoning with inconsistent knowledge. We also show that the generative model gives a new classification algorithm that outperforms several representative algorithms in predictive accuracy and complexity on the Kaggle Titanic dataset.
\end{abstract}

\section{Introduction}
Bayes' theorem plays an important role today in various fields such as AI, neuroscience, cognitive science, statistical physics and bioinformatics. It underlies most modern approaches to uncertain reasoning in AI systems \cite{russell:09}. In neuroscience, it is often successfully used as a metaphor for functions of the cerebral cortex, which is the outer portion of the brain in charge of higher-order functions such as perception, memory, emotion and thought \cite{lee:03,knill:04,george:05,colombo:12,funamizu:16}. These successes of Bayesian methods give rise to the Bayesian brain hypothesis that the brain is a Bayesian machine \cite{friston:12,sanborn:16}.
\par
Logic concerns entailment (i.e. a consequence relation) whereas learning concerns prediction. They are both practices of the human brain. The Bayesian brain hypothesis thus leads to another hypothesis that there is a common Bayesian interpretation of entailment and prediction, which are traditionally studied in different disciplines. The interpretation is important for the following reasons. First, it gives a more unified view to critically assess the existing formalisms of entailment and prediction. Second, it has a potential to give a better explanation of how the human brain performs them. Third, it backs up the Bayesian brain hypothesis emerging from the field of neuroscience. In spite of the values, few research has focused on the unified interpretation in terms of Bayesian perspectives (see Section \ref{sec:discussion}).
%
%
\par
%
In this paper, we give a formal account of the process of how the truth value of a sentence is probabilistically generated from the probability distribution over states of the world. Our model based on this idea, often called a generative model, begins by assuming a probability distribution over states of the world, e.g. valuation functions in propositional logic. The probability of each state of the world represents how much it is natural, normal or typical. We then formalise the causal relation between each state of the world and each sentence. Let $w$ and $\alpha$ denote a state of the world and a sentence, respectively. The probability that $\alpha$ is true, denoted by $p(\alpha)$, will be shown to have
%
\begin{eqnarray*}
p(\alpha)=\sum_{w}p(\alpha,w)=\sum_{w}p(\alpha|w)p(w).
\end{eqnarray*}
The equation states that the probability of the truth value of $\alpha$ is the weighted average of the products of likelihood $p(\alpha|w)$ and prior $p(w)$ over all states of the world. Given a set $\Delta$ of sentences, we will show to have
%
%
\begin{eqnarray*}
p(\alpha|\Delta)=\sum_{w}p(\alpha|w)p(w|\Delta).
\end{eqnarray*}
This equation is known as a form of Bayesian learning \cite{russell:09}. It states that the probability of the truth value of $\alpha$ is the weighted average of the products of likelihood $p(\alpha|w)$ and posterior $p(w|\Delta)$ over all states of the world.
%
\par
We define Bayesian entailment using a conditional probability with a fixed probability threshold. Several important logical and machine learning properties are derived from the simple idea. The Bayesian entailment is shown to be identical to the classical consequence relation in reasoning with consistent knowledge. In addition, it is a paraconsistent consequence relation in reasoning with inconsistent knowledge, and it is a nonmonotonic consequence relation in deterministic situations. We moreover show that the Bayesian entailment outperforms several representative classification algorithms in predictive accuracy and complexity on the Kaggle Titanic dataset.
\par
This paper contributes to the field of commonsense reasoning by providing a simple inference principle that is correct in terms of classical logic, paraconsistent logic, nonmonotonic logic and machine learning. It gives a more general answer to the questions such as how to logically infer from inconsistent knowledge, how to rationally handle defeasibility of everyday reasoning, and how to probabilistically infer from noisy data without a conditional dependence assumption, which are all studied and explained individually.
\par
This paper is organised as follows. Section 2 gives a simple generative model for a Bayesian consequence relation. Section 3 shows logical and machine learning correctness of the generative model. Section 4 concludes with discussion of related work.
%
\section{Method}
We assume a syntax-independent logical language, denoted by $L$. It is logical in the sense that it is defined using only usual logical connectives such as $\lnot$, $\land$, $\lor$, $\rightarrow$, $\leftarrow$ and $\leftrightarrow$. It is syntax independent in the sense that it specifies no further syntax such as propositional or first-order language.
\par
An interpretation is an assignment of truth values to well-formed formulas. It is given by a valuation function in propositional logic, and is given by a structure and variable assignment in first-order logic. In this paper, we call them a possible world to make our discussion general. We assume a probability distribution over possible worlds to quantify the uncertainty of each possible world. Let $W$ denote a random variable for possible worlds, $w_{i}$ the $i$-th possible world, and $\phi_{i}$ the probability of the occurrence of $w_{i}$, i.e., $p(W=w_{i})=\phi_{i}$. Then, the probability distribution over possible worlds can be modelled as a categorical distribution with parameter $(\phi_{1},\phi_{2},...,\phi_{N})$ where $\sum_{i=1}^{N}\phi_{i}=1$ and $\phi_{i}\in[0,1]$, for all $i$. That is, we have
\begin{eqnarray*}
p(W)=(\phi_{1},\phi_{2},...,\phi_{N}).
\end{eqnarray*}
%
We assume that its prior distribution is statistically estimated from data. For all natural numbers $i$ and $j$, $\phi_{i}>\phi_{j}$ intuitively means that the interpretation specified by possible world $w_{i}$ is more natural, typical or normal than that of $w_{j}$, according to given data.
\par
In formal logic, truth values of formulas depend on possible worlds. The interpretation uniquely given in each possible world indeed assigns a certain truth value to every formula. In this paper, we consider the presence of noise in interpretation. We assume that every formula is a random variable whose realisations are 0 and 1, meaning false and true, respectively. Variable $\mu\in[0,1]$ denotes the probability that a formula is interpreted as being true (resp. false) in a possible world when it is actually true (resp. false) in the same possible world. $1-\mu$ is thus the probability that a formula is interpreted as being true (resp. false) in a possible world when it is actually false (resp. true) in the same possible world. For any possible worlds $w$ and formulas $\alpha$, we thus define the conditional probability of each truth value of $\alpha$ given $w$, as follows.
\begin{eqnarray*}
p(\alpha=1|W=w)=
\begin{cases}
\mu & \text{if } w\in\llbracket\alpha=1\rrbracket\\
1-\mu & \text{otherwise }
\end{cases}
\\%
p(\alpha=0|W=w)=
\begin{cases}
\mu & \text{if } w\in\llbracket\alpha=0\rrbracket\\
1-\mu & \text{otherwise }
\end{cases}
\end{eqnarray*}
Here, $\llbracket\alpha=1\rrbracket$ denotes the set of all possible worlds in which $\alpha$ is true, and $\llbracket\alpha=0\rrbracket$ the set of all possible worlds in which $\alpha$ is false. The above expressions can be simply written as a Bernoulli distribution with parameter $\mu$ where $0\leq \mu\leq 1$. That is, we have
%
%
\begin{eqnarray*}
p(\alpha|W=w)=\mu^{\llbracket\alpha\rrbracket_{w}}(1-\mu)^{1-\llbracket\alpha\rrbracket_{w}}.
\end{eqnarray*}
%
%
Here, $\llbracket\alpha\rrbracket$ is either $\llbracket\alpha=0\rrbracket$ or $\llbracket\alpha=1\rrbracket$, and $\llbracket\alpha\rrbracket_{w}$ denotes a function of $w$ and $\alpha$ that returns 1 if $w\in\llbracket\alpha\rrbracket$ and 0 otherwise.
\par
In formal logic, the truth values of formulas are independently determined from each possible world. In probabilistic terms, the truth values of any two formulas $\alpha_{1}$ and $\alpha_{2}$ are conditionally independent given a possible world $w$, i.e., $p(\alpha_{1},\alpha_{2}|w)=p(\alpha_{1}|w)p(\alpha_{2}|w)$\footnote{Note that this equation holds not only for atomic formulas but also for compound formulas. The independence, i.e., $p(\alpha_{1},\alpha_{2})=p(\alpha_{1})p(\alpha_{2})$, however, holds only for atomic formulas.}. Let $\Delta=\{\alpha_{1},\alpha_{2},...,\alpha_{N}\}$ be the set of $N$ formulas. We thus have
%
\begin{eqnarray*}
p(\Delta|W=w)=\prod_{n=1}^{N}p(\alpha_{n}|W=w).
\end{eqnarray*}
%
%
%
%
\par
So far, we defined prior distribution $p(W)$ as a categorical distribution with parameter $(\phi_{1},\phi_{2},...,\phi_{N})$ and model likelihood $p(\Delta|W)$ as Bernoulli distributions with parameter $\mu$. Given all of the parameters, they give the full joint distribution over all of the random variables. We call $\{p(\Delta|W)$, $p(W)\}$ the probabilistic-logical model, or simply the logical model. When the parameters of the logical model need to be specified, we write the logical model as $\{p(\Delta|W,\mu)$, $p(W|\phi_{1},\phi_{2},...,\phi_{N})\}$.
%
%
%
%
\par
Now, let $Pow(L)$ denote the powerset of logical language $L$. On the logical model, we define a consequence relation called a Bayesian entailment.
\begin{definition}[Bayesian entailment]\label{def:BE}
Let $\theta\in[0,1]$. $\vapprox_{\theta}\subseteq Pow(L)\times L$ is a Bayesian entailment with probability threshold $\theta$ if $\Delta\vapprox_{\theta}\alpha$ holds if and only if $p(\alpha|\Delta)\geq \theta$ holds. 
\end{definition}
It is obvious from the definition that $\vapprox_{\theta_{1}}\subseteq\vapprox_{\theta_{2}}$ holds, for all $\theta_{1}\in[0,1]$ and $\theta_{2}\in[0,\theta_{1}]$.
\par
The Bayesian entailment is actually Bayesian in the sense that it involves the following form of Bayesian learning where the probability of consequence $\alpha$ is weighted averages over the posterior distribution of all possible worlds in which premise $\Delta$ is true.
\begin{eqnarray*}
p(\alpha |\Delta)=\sum_{w}p(\alpha|w,\Delta)p(w|\Delta)=\sum_{w}p(\alpha|w)p(w|\Delta)
\end{eqnarray*} 
Therefore, the Bayesian entailment is an application of Bayesian prediction on the logical model.
\par
On the logical model, we also define a consequence relation called a maximum a posteriori (MAP) entailment.
\begin{definition}[Maximum a posteriori entailment]
$\vapprox_{MAP}$ $\subseteq$ $Pow(L) $ $\times$ $L$ is a maximum a posteriori entailment if $\Delta\vapprox_{MAP}\alpha$ holds if and only if there is $w_{MAP}\in\argmax_{w}p(w|\Delta)$ such that $w_{MAP}\in\llbracket\alpha\rrbracket$.
\end{definition}
Here, $w_{MAP}\in\argmax_{w}p(w|\Delta)$ is said to be a maximum a posteriori estimate. It is intuitively the most likely possible world given $\Delta$. The maximum a posteriori entailment can be seen as an approximation of the Bayesian entailment. They are equivalent under the assumption that posterior distribution $p(W|\Delta)$ has a sharp peak, meaning that a possible world is very normal, natural or typical. Under the assumption, we have $p(W|\Delta)\simeq 1$ if $W=w_{MAP}$ and $0$ otherwise, where $\simeq$ denotes an approximation. We thus have
\begin{eqnarray*}
p(\alpha|\Delta)&=&\sum_{w}p(\alpha|w)p(w|\Delta)\\
&\simeq& p(\alpha|w_{MAP})=
\begin{cases}
\mu & (w_{MAP}\in\llbracket\alpha\rrbracket)\\
1-\mu & (w_{MAP}\notin\llbracket\alpha\rrbracket)
\end{cases}
\end{eqnarray*}
Note that both the Bayesian entailment and the maximum a posteriori entailment are general in the sense that the parameters, i.e., $\mu$ and $(\phi_{1},\phi_{2},...,\phi_{N})$, of the logical model are all unspecified.
%
\par
The probability of the truth value of each formula is not primitive in the logical model. We thus guarantee that it satisfies the Kolmogorov axioms.
%
\begin{proposition}\label{kolmogorov}
Let $\alpha,\beta\in L$.
\begin{enumerate}
\item $0\leq p(\alpha=i)$ holds, for all $i\in\{0,1\}$.
\item $\sum_{i\in\{0,1\}}p(\alpha=i)=1$ holds.
\item $p(\alpha\lor\beta=i)=p(\alpha=i)+p(\beta=i)-p(\alpha\land\beta=i)$ holds, for all $i\in\{0,1\}$.
\end{enumerate}
\end{proposition}
\begin{proof}
See Appendix.
\end{proof}
\par
The next proposition shows that the logical model is sound in terms of logical negation.
\begin{proposition}\label{negation}
For all $\alpha\in L$, $p(\alpha=0)=p(\neg\alpha=1)$ holds.
\end{proposition}
\begin{proof}
See Appendix.
\end{proof}
In what follows, we thus replace $\alpha=0$ by $\lnot\alpha=1$ and then abbreviate $\lnot\alpha=1$ to $\lnot\alpha$.
%
Now, let's see an example in propositional logic.
%
\begin{table}[t]
\caption{Possible-world distribution and truth-value likelihoods.}
\label{ex.Int}
\begin{center}
\begin{tabular}{c|c|cc|cc}
%
& $p(W)$ & $rain$ & $wet$ & $p(rain|W)$ & $p(wet|W)$\\\hline
$w_{1}$ & $0.4$ & $0$ & $0$ & $1-\mu$ & $1-\mu$\\
$w_{2}$ & $0.2$ & $0$ & $1$ & $1-\mu$ & $\mu$\\
$w_{3}$ & $0.1$ & $1$ & $0$ & $\mu$ & $1-\mu$\\
$w_{4}$ & $0.3$ & $1$ & $1$ & $\mu$ & $\mu$
\end{tabular}
\end{center}
\end{table}
%
\begin{example}\label{ex:BE}
Let $rain$ and $wet$ be two propositional symbols meaning ``it is raining'' and ``the grass is wet'', respectively. The second column of Table \ref{ex.Int} shows the probability distribution over all valuation functions. The fifth and sixth columns show the likelihoods of the atomic propositions being true given a valuation function. Given $\mu=1$, predictive probability $p(rain|wet)$ is calculated as follows.
\begin{eqnarray*}
p(rain|wet)&=&\frac{\sum_{w}p(w)p(rain|w)p(wet|w)}{\sum_{w}p(w)p(wet|w)}\\
&=&\frac{\mu^{2}\phi_{4}+\mu(1-\mu)(\phi_{2}+\phi_{3})+(1-\mu)^{2}\phi_{1}}{\mu(\phi_{2}+\phi_{4})+(1-\mu)(\phi_{1}+\phi_{3})}\\
&=&\frac{0.3\mu^{2}+(0.2+0.1)\mu(1-\mu)+0.4(1-\mu)^{2}}{(0.2+0.3)\mu+(0.4+0.1)(1-\mu)}\\
&=&\frac{0.4\mu^{2}-0.5\mu+0.4}{0.5}=0.6
\end{eqnarray*}
Therefore, $\{wet\}\vapprox_{\theta} rain$ thus holds, for all $\theta\leq 0.6$. Figure \ref{dependency2} shows the Bayesian network visualising the dependency of the random variables and parameters used in this calculation.
\end{example}
%
\begin{figure}[t]
\begin{center}
 \includegraphics[scale=0.35]{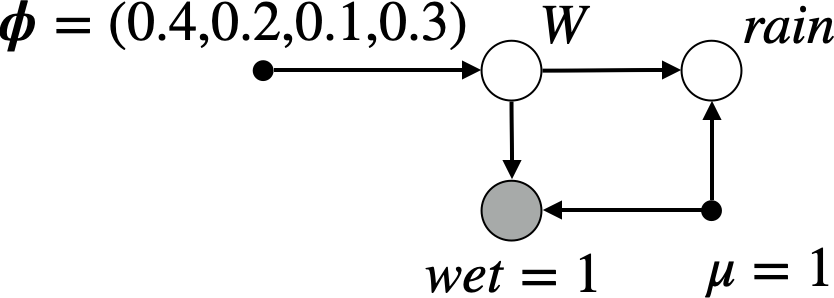}
  \caption{Bayesian network visualising the dependency of the elements of the probabilistic model.}
  \label{dependency2}
\end{center}
\end{figure}
%
%
%
\section{Correctness}
This section discusses logical and machine learning correctness of the logical model. The logical model is specialised in several ways to show that the Bayesian entailments defined on the specialised models perform key logical and machine learning tasks.
%

\subsection{Classicality}
Recall that a set $\Delta$ of formulas entails a formula $\alpha$ in classical logic, denoted by $\Delta\models\alpha$, if and only if $\alpha$ is true in every possible world in which $\Delta$ is true. In this paper, we call the Bayesian entailment defined on the logical model $\{p(\Delta|W,\mu=1),p(W|\phi_{1}=1/N,\phi_{2}=1/N,...,\phi_{N}=1/N)\}$ the Bayesian classical entailment. The model can be seen as an ideal specialisation of the logical model in the absence of data and noise. Each formula is interpreted without noise effect, i.e., $\mu=1$, in possible worlds that are equally likely, i.e., $(\phi_{1}=1/N,\phi_{2}=1/N,...,\phi_{N}=1/N)$. The following two theorems state that the Bayesian classical entailment $\vapprox_{1}$ is a proper fragment of the classical entailment, i.e., $\vapprox_{1}\subseteq\models$.
\begin{theorem}\label{thrm:1}
Let $\alpha\in L$, $\Delta\subseteq L$ and $\vapprox_{1}$ be the Bayesian classical entailment. If there is a model of $\Delta$ then $\Delta\vapprox_{1}\alpha$ if and only if $\Delta\models\alpha$.
%
\end{theorem}
\begin{proof}
Let $|\Delta|$ denote the cardinality of $\Delta$. Dividing possible worlds into the models of $\Delta$ and the others, we have
\begin{eqnarray*}
p(\alpha|\Delta)&=&\frac{\displaystyle{\sum_{w}p(\alpha|w)p(\Delta|w)p(w)}}{\displaystyle{\sum_{w}p(\Delta|w)p(w)}}\\
&=&\frac{\displaystyle{\sum_{w\in\llbracket\Delta\rrbracket}p(w)p(\alpha|w)\mu^{|\Delta|}+\sum_{w\notin\llbracket\Delta\rrbracket}p(w)p(\alpha|w)p(\Delta|w)}}{\displaystyle{\sum_{w\in\llbracket\Delta\rrbracket}p(w)\mu^{|\Delta|}+\sum_{w\notin\llbracket\Delta\rrbracket}p(w)p(\Delta|w)}}.
\end{eqnarray*}
$p(\Delta|w)=\prod_{\beta\in\Delta}p(\beta|w)=\prod_{\beta\in\Delta}\mu^{\llbracket\beta\rrbracket_{w}}(1-\mu)^{1-{\llbracket\beta\rrbracket_{w}}}$. For all $w\notin\llbracket\Delta\rrbracket$, there is $\beta\in\Delta$ such that $\llbracket\beta\rrbracket_{w}=0$. Thus, $p(\Delta|w)=0$ when $\mu=1$, for all $w\notin\llbracket\Delta\rrbracket$. We thus have
\begin{eqnarray*}
p(\alpha|\Delta)&=&\frac{\sum_{w\in\llbracket\Delta\rrbracket}p(w)\mu^{\llbracket\alpha\rrbracket_{w}}(1-\mu)^{1-\llbracket\alpha\rrbracket_{w}}\mu^{|\Delta|}}{\sum_{w\in\llbracket\Delta\rrbracket}p(w)\mu^{|\Delta|}}\\
&=&\frac{\sum_{w\in\llbracket\Delta\rrbracket}p(w)\llbracket\alpha\rrbracket_{w}}{\sum_{w\in\llbracket\Delta\rrbracket}p(w)}=\frac{\sum_{w\in\llbracket\alpha,\Delta\rrbracket}p(w)}{\sum_{w\in\llbracket\Delta\rrbracket}p(w)}.
%
\end{eqnarray*}
Now, $p(\alpha|\Delta)=\frac{\sum_{w\in\llbracket\alpha,\Delta\rrbracket}p(w)}{\sum_{w\in\llbracket\Delta\rrbracket}p(w)}=1$ if and only if $\llbracket\alpha\rrbracket\supseteq\llbracket\Delta\rrbracket$, i.e., $\Delta\models\alpha$.
\end{proof}
\begin{theorem}\label{thrm:2}
Let $\alpha\in L$, $\Delta\subseteq L$ and $\vapprox_{1}$ be the Bayesian classical entailment. If there is no model of $\Delta$ then $\Delta\vapprox_{1}\alpha$ implies $\Delta\models\alpha$, but not vice versa.
\end{theorem}
\begin{proof}
($\Rightarrow$) If $\llbracket\Delta\rrbracket=\emptyset$ then $\Delta\models\alpha$, for all $\alpha$, in classical logic. ($\Leftarrow$) Definition \ref{def:BE} implies that $\Delta\vapprox_{\theta}\alpha$ if $p(\alpha|\Delta)\geq\theta$ holds, and $\Delta\not\vapprox_{\theta}\alpha$ if $p(\alpha|\Delta)<\theta$ holds or $p(\alpha|\Delta)$ is undefined. Given $\Delta=\{\beta,\lnot\beta\}$, the following derivation exemplifies that the predictive probability of a formula $\alpha$ is undefined due to division by zero.
\begin{eqnarray*}
p(\alpha|\beta,\lnot\beta)&=&\frac{\sum_{w}p(w)p(\alpha|w)p(\beta|w)p(\lnot\beta|w)}{\sum_{w}p(w)p(\beta|w)p(\lnot\beta|w)}\\
&=&\frac{\mu(1-\mu)\sum_{w}p(w)p(\alpha|w)}{\mu(1-\mu)\sum_{w}p(w)}~~~(\text{undefined if }\mu=1)
\end{eqnarray*}
\end{proof}
In classical logic, everything can be entailed from a contradiction. However, Theorem \ref{thrm:2} implies that nothing can be entailed from a contradiction using the Bayesian classical entailment. In the next section, we study a logical model that allows us to derive something useful from a contradiction.
%
\subsection{Paraconsistency}
In classical logic, the presence of contradictions in a knowledge base and the fact that the knowledge base entails everything are inseparable. In practice, this fact calls for truth maintenance of the knowledge base, which makes it difficult to scale up the knowledge base toward a useful AI application beyond toy problems.
\par
%
In this section, we consider the logical model with specific parameters such that $\mu$ approaches 1 and $(\phi_{1},\phi_{2},...,\phi_{N})$ is a uniform distribution, i.e, $\mu\rightarrow 1$ and $\phi_{n}=1/N$, for all $n$. Then, the specific logical model is written as $\{\lim_{\mu\rightarrow 1}$ $p(\Delta|W,\mu)$, $p(W|\phi_{1}=1/N$, $\phi_{2}=1/N$ ,..., $\phi_{N}=1/N)\}$. We call the Bayesian entailment defined on the logical model the Bayesian paraconsistent entailment. Similar to the classical model, the model is an ideal specialisation of the logical model in the absence of data, where formulas are interpreted without noise effect in every possible world that is equally likely.
\par
The following two theorems state that the Bayesian paraconsistent entailment $\vapprox_{1}$ is also a proper fragment of the classical entailment, i.e., $\vapprox_{1}\subseteq\models$.
\begin{theorem}\label{thrm:3}
Let $\alpha\in L$, $\Delta\subseteq L$ and $\vapprox_{1}$ be the Bayesian paraconsistent entailment. If there is a model of $\Delta$ then $\Delta\vapprox_{1}\alpha$ if and only if $\Delta\models\alpha$.
\end{theorem}
\begin{proof}
The proof of Theorem \ref{thrm:1} still holds under the presence of the limit operation.
\end{proof}
\begin{theorem}\label{thrm:4}
Let $\alpha\in L$, $\Delta\subseteq L$ and $\vapprox_{1}$ be the Bayesian paraconsistent entailment. If there is no model of $\Delta$ then $\Delta\vapprox_{1}\alpha$ implies $\Delta\models\alpha$, but not vice versa.
\end{theorem}
\begin{proof}
($\Rightarrow$) The proof of Theorem \ref{thrm:2} still holds. ($\Leftarrow$) Suppose $p(\alpha)<1$. The following derivation exemplifies $p(\alpha|\beta\land\lnot\beta)<1$.
\begin{flalign*}
p(\alpha|\beta\land\lnot\beta)&=\frac{\sum_{w}p(w)\lim_{\mu\rightarrow 1}p(\alpha|w)\lim_{\mu\rightarrow 1}p(\beta\land\lnot\beta|w)}{\sum_{w}p(w)\lim_{\mu\rightarrow 1}p(\beta\land\lnot\beta|w)}\\
&=\lim_{\mu\rightarrow 1}\frac{(1-\mu)\sum_{w}p(w)p(\alpha|w)}{(1-\mu)\sum_{w}p(w)}=\lim_{\mu\rightarrow 1}\frac{\sum_{w}p(w)p(\alpha|w)}{\sum_{w}p(w)}\\
&=\sum_{w}p(w)\lim_{\mu\rightarrow 1}p(\alpha|w)=p(\alpha)
\end{flalign*}
\end{proof}
\par
The Bayesian paraconsistent entailment handles reasoning with inconsistent knowledge in a proper way described below. Let $\alpha\in L$ and $\Delta\subseteq L$. For simplicity, we use symbol $\#_{w}$ to denote the number of formulas in $\Delta$ that are true in $w$, i.e. $\#_{w}=\sum_{\beta\in\Delta}\llbracket\beta\rrbracket_{w}$, and symbol $(\!(\Delta)\!)$ to denote the set of possible worlds in which the maximum number of formulas in $\Delta$ are true, i.e., $(\!(\Delta)\!)=\argmax_{w}(\#_{w})$. $(\!(\Delta)\!)$ is thus the set of models of $\Delta$, i.e., $(\!(\Delta)\!)=\llbracket\Delta\rrbracket$, if and only if there is a model of $\Delta$. Now, by case analysis of the possible worlds in $(\!(\Delta)\!)$ and the others, we have 
%
\begin{flalign*}
p(\alpha|\Delta)&=\lim_{\mu\rightarrow 1}\frac{\sum_{w}p(\alpha|w)p(w)p(\Delta|w)}{\sum_{w}p(w)p(\Delta|w)}\\
&=\lim_{\mu\rightarrow 1}\frac{\displaystyle{\sum_{\hat{w}\in(\!(\Delta)\!)}p(\alpha|\hat{w})p(\hat{w})p(\Delta|\hat{w})+\sum_{w\notin(\!(\Delta)\!)}p(\alpha|w)p(w)p(\Delta|w)}}{\displaystyle{\sum_{\hat{w}\in(\!(\Delta)\!)}p(\hat{w})p(\Delta|\hat{w})+\sum_{w\notin(\!(\Delta)\!)}p(w)p(\Delta|w)}}.
\end{flalign*}
%
$p(\Delta|w)=\prod_{\beta\in\Delta}p(\beta|w)=\mu^{\#_{w}}(1-\mu)^{|\Delta|-\#_{w}}$ holds, for all $w$. Since $\#_{\hat{w}}$ has the same value for all $\hat{w}\in(\!(\Delta)\!)$, we can simplify the fraction by dividing the denominator and numerator by $(1-\mu)^{|\Delta|-\#_{\hat{w}}}$. The fraction inside of the limit operator is now given by
\begin{eqnarray*}
\frac{\displaystyle{\sum_{\hat{w}\in(\!(\Delta)\!)}p(\alpha|\hat{w})p(\hat{w})\mu^{\#_{\hat{w}}}+\sum_{w\notin(\!(\Delta)\!)}p(\alpha|w)p(w)\mu^{\#_{w}}(1-\mu)^{\#_{\hat{w}}-\#_{w}}}}{\displaystyle{\sum_{\hat{w}\in(\!(\Delta)\!)}p(\hat{w})\mu^{\#_{\hat{w}}}+\sum_{w\notin(\!(\Delta)\!)}p(w)\mu^{\#_{w}}(1-\mu)^{\#_{\hat{w}}-\#_{w}}}}.
\end{eqnarray*}
Applying the limit operation to the second terms of the denominator and numerator, we have
\begin{eqnarray*}
p(\alpha|\Delta)&=&\lim_{\mu\rightarrow 1}\frac{\sum_{\hat{w}\in(\!(\Delta)\!)}p(\alpha|\hat{w})p(\hat{w})\mu^{\#_{\hat{w}}}}{\sum_{\hat{w}\in(\!(\Delta)\!)}p(\hat{w})\mu^{\#_{\hat{w}}}}\\
&=&\lim_{\mu\rightarrow 1}\frac{\sum_{\hat{w}\in(\!(\Delta)\!)}\mu^{\llbracket\alpha\rrbracket_{\hat{w}}}(1-\mu)^{1-\llbracket\alpha\rrbracket_{\hat{w}}}p(\hat{w})\mu^{\#_{\hat{w}}}}{\sum_{\hat{w}\in(\!(\Delta)\!)}p(\hat{w})\mu^{\#_{\hat{w}}}}=\frac{\sum_{\hat{w}\in(\!(\Delta)\!)}\llbracket\alpha\rrbracket_{\hat{w}}p(\hat{w})}{\sum_{\hat{w}\in(\!(\Delta)\!)}p(\hat{w})}
\end{eqnarray*}
From the above derivation, $\Delta\vapprox_{1}\alpha$ holds if and only if $\llbracket\alpha\rrbracket\supseteq (\!(\Delta)\!)$. For the sake of intuition, let us say that $\Delta$ is almost true in a possible world $w$ if $w\in(\!(\Delta)\!)\setminus\llbracket\Delta\rrbracket$. Then, $\Delta\vapprox_{1}\alpha$ states that
\begin{itemize}
\item if $\Delta$ has a model then $\alpha$ is true in every possible world in which $\Delta$ is true, i.e., $\Delta\models\alpha$, and
\item if $\Delta$ has no model then $\alpha$ is true in every possible world in which $\Delta$ is almost true.
\end{itemize}
%
\begin{example}
Let $a$ and $b$ be two propositional variables. Suppose the uniform prior distribution over possible worlds of the two variables.
\begin{itemize}
\item $p(a|a,b,\lnot b)=1$ holds. Therefore, $a,b,\lnot b\vapprox_{\theta}a$ holds if and only if $\theta=1$. 
\item $p(a|a\land b,\lnot b)=\frac{2}{3}$ holds. Therefore, $a\land b,\lnot b\vapprox_{\theta}a$ holds if and only if $\theta\leq \frac{2}{3}$.
\item $p(a|a\land b\land\lnot b)=\frac{1}{2}$ holds. Therefore, $a\land b\land\lnot b\vapprox_{\theta}a$ holds if and only if $\theta\leq \frac{1}{2}$.
\end{itemize}
\end{example}
\par
Let us examine abstract inferential properties of the Bayesian paraconsistent entailment. Mathematically, let $\alpha,\beta\in L$, $\Delta\subseteq L$ and $\vdash$ be a consequence relation over logical language $L$, i.e., $\vdash\subseteq Pow(L)\times L$. We call tuple $(L,\vdash)$ a logic. A logic is said to be non-contradictory, non-trivial, and explosive if it satisfies the following respective principles.
%
\begin{itemize}
\item Non-contradiction: $\exists\Delta\forall\alpha(\Delta\not\vdash\alpha~or~\Delta\not\vdash\lnot\alpha)$
\item Non-triviality: $\exists\Delta\exists\alpha(\Delta\not\vdash\alpha)$
\item Explosion: $\forall\Delta\forall\alpha\forall\beta(\Delta,\alpha,\lnot\alpha\vdash\beta)$
\end{itemize}
A logic is paraconsistent if and only if it is not explosive, and is sometimes called dialectical if it is contradictory \cite{carnielli}. The following theorem states that the Bayesian paraconsistent entailment is paraconsistent, but not dialectical.
%
%
%
%
\begin{theorem}
Let $\theta\in(0.5,1]$. The Bayesian paraconsistent entailment $\vapprox_{\theta}$ satisfies the principles of non-contradiction and non-triviality, but does not satisfy the principle of explosion.
\end{theorem}
\begin{proof}
(1) It is sufficient to show $\nexists\alpha(\vapprox_{\theta}\alpha$ and $\vapprox_{\theta}\lnot\alpha)$. From definition, we show there is no $\alpha$ such that $p(\alpha)\geq\theta$ and $p(\lnot\alpha)\geq\theta$. We have
\begin{flalign*}
p(\alpha)&=\sum_{w}\lim_{\mu\rightarrow 1}p(\alpha|w)p(w)=\sum_{w}\lim_{\mu\rightarrow 1}\mu^{\llbracket\alpha\rrbracket_{w}}(1-\mu)^{1-\llbracket\alpha\rrbracket_{w}}p(w)\\
&=\lim_{\mu\rightarrow 1}\{\mu\sum_{w\in\llbracket\alpha\rrbracket} p(w)+(1-\mu)\sum_{w\notin\llbracket\alpha\rrbracket} p(w)\}=\sum_{w\in\llbracket\alpha\rrbracket} p(w)\\
p(\lnot\alpha)&=\sum_{w}\lim_{\mu\rightarrow 1}p(\lnot\alpha|w)p(w)=\sum_{w}\lim_{\mu\rightarrow 1}\mu^{1-\llbracket\alpha\rrbracket_{w}}(1-\mu)^{\llbracket\alpha\rrbracket_{w}}p(w)\\
&=\lim_{\mu\rightarrow 1}\{\mu\sum_{w\notin\llbracket\alpha\rrbracket} p(w)+(1-\mu)\sum_{w\in\llbracket\alpha\rrbracket} p(w)\}=\sum_{w\notin\llbracket\alpha\rrbracket} p(w)
\end{flalign*}
Now, $p(\alpha)+p(\lnot\alpha)=\sum_{w\in\llbracket\alpha\rrbracket} p(w)+\sum_{w\notin\llbracket\alpha\rrbracket} p(w)=\sum_{w}p(w)=1$.
(2) It is sufficient to show $\exists\alpha\not\vapprox_{\theta}\alpha$. We show there is $\alpha$ such that $p(\alpha)\leq 0.5$. Using proof by contradiction, we assume $p(\alpha)> 0.5$ holds, for all $\alpha$. This contradicts (1).
(3) It is sufficient to show $\exists\alpha\exists\beta(\alpha,\lnot\alpha\not\vapprox_{\theta}\beta)$ holds. $p(\beta|\alpha,\lnot\alpha)=p(\beta)$ is shown as follows.
\begin{eqnarray*}
p(\beta|\alpha,\lnot\alpha)&=&\lim_{\mu\rightarrow 1}\frac{\sum_{w}p(\alpha|w)p(\lnot\alpha|w)p(\beta|w)p(w)}{\sum_{w}p(\alpha|w)p(\lnot\alpha|w)p(w)}\\
&=&\lim_{\mu\rightarrow 1}\frac{\mu(1-\mu)\sum_{w}p(\beta|v)p(w)}{\mu(1-\mu)\sum_{w}p(w)}=p(\beta)
\end{eqnarray*}
The principle of explosion does not hold when $p(\beta)<1$.
\end{proof}

\subsection{Non-monotonicity}
In classical logic, whenever a sentence is a logical consequence of a set of sentences, then the sentence is also a consequence of an arbitrary superset of the set. This property called monotonicity cannot be expected in commonsense reasoning where having new knowledge often invalidates a conclusion. A practical knowledge-based system with this property is possible under the unrealistic assumption that every rule in the knowledge base sufficiently covers possible exceptions.
%
\par
A preferential entailment \cite{shoham:87} is a general approach to a nonmonotonic consequence relation. It is defined on a preferential structure $({\cal W},\succ)$, where ${\cal W}$ is a set of valuation functions of propositional logic and $\succ$ is an irreflexive and transitive relation on ${\cal W}$. $w_{1}\succ w_{2}$ represents that $w_{1}$ is preferable\footnote{For the sake of simplicity, we do not adopt the common practice in logic that $w_{2}\succ w_{1}$ denotes $w_{1}$ is preferable to $w_{2}$.} to $w_{2}$ in the sense that $w_{1}$ is more normal, typical or natural than $w_{2}$. Given a preferential structure $({\cal W},\succ)$, $\alpha$ is preferentially entailed by $\Delta$, denoted by $\Delta\vsim_{({\cal W},\succ)}\alpha$, if $\alpha$ is true in all $\succ$-maximal\footnote{$\succ$ has to be smooth (or stuttered) \cite{kraus:90} so that a maximal model certainly exists.} models of $\Delta$.
\par
Given a preferential structure $({\cal W},\succ)$, we consider the logical model with specific parameters $\mu\rightarrow 1$ and $(\phi_{1},\phi_{2},...,\phi_{N})$ such that, for all $w_{1}$ and $w_{2}$ in ${\cal W}$, if $w_{1}\succ w_{2}$ then $\phi_{1}\geq\phi_{2}$.\footnote{If we assume a function mapping $w_{n}$ to $\phi_{n}$, for all $n$, then the function satisfying the condition is said to be order-preserving.} We call the maximum a posteriori entailment defined on the logical model the maximum a posteriori entailment with respect to $({\cal W}, \succ)$. The following two theorems show the relationship between the maximum a posteriori entailment and preferential entailment.
\begin{theorem}\label{thrm:MAP1}
Let $({\cal W},\succ)$ be a preferential structure and $\vapprox_{MAP}$ be a maximum a posteriori entailment with respect to $({\cal W},\succ)$. If there is a model of $\Delta$ then $\Delta\vsim_{({\cal W},\succ)}\alpha$ implies $\Delta\vapprox_{MAP}\alpha$.
\end{theorem}
\begin{proof}
Since $\geq$ is a linear extension of $\succ$ given ${\cal W}$, if $w_{1}\succ w_{2}$ then $\phi_{1}\geq \phi_{2}$, for all $w_{1}, w_{2}\in{\cal W}$. Thus, if $w_{i}$ is $\succ$-maximal then $\phi_{i}$ is maximal or there is another $\succ$-maximal $w_{j}$ such that $\phi_{j}\geq \phi_{i}$. Therefore, there is $w^{*}$ such that $w^{*}$ is a $\succ$-maximal model of $\Delta$ and $w^{*}\in\argmax_{w}p(w|\Delta)$. $\alpha$ is true in $w^{*}$ since $\Delta\vsim_{({\cal W},\succ)}\alpha$.
\end{proof}
\begin{theorem}\label{thrm:MAP2}
Let $({\cal W},\succ)$ be a preferential structure and $\vapprox_{MAP}$ be a maximum a posteriori entailment with respect to $({\cal W},\succ)$. If there is no model of $\Delta$ then $\Delta\vapprox_{MAP}\alpha$ implies $\Delta\vsim_{({\cal W},\succ)}\alpha$, but not vice versa.
\end{theorem}
\begin{proof}
($\Rightarrow$) From the definition, $\Delta\vsim_{({\cal W},\succ)}\alpha$ holds, for all $\alpha$, when $\Delta$ has no model. ($\Leftarrow$) Let $\alpha,\beta\in L$. Suppose $(\phi_{1},\phi_{2},..,\phi_{N})$ such that $w_{1}\notin\llbracket\alpha\rrbracket$ and $\phi_{n}>\phi_{n+1}$, for all $1\leq n\leq N-1$. Now, $p(W|\beta,\lnot\beta)=p(W)$ is shown as follows.
\begin{eqnarray*}
p(W|\beta,\lnot\beta)&=&\frac{p(\beta|W)p(\lnot\beta|W)p(W)}{\sum_{w}p(\beta|w)p(\lnot\beta|w)p(w)}\\
&=&\frac{\mu(1-\mu)p(W)}{\mu(1-\mu)\sum_{w}p(w)}=p(W)
\end{eqnarray*}
Although $w_{1}=\argmax_{w}p(w|\beta,\lnot\beta)$, $w_{1}\notin\llbracket\alpha\rrbracket$.
\end{proof}
\par
When a preferential structure is assumed to be a total order, the maximum a posteriori entailment with respect to the preferential structure becomes a fragment of the preferential entailment.
\begin{theorem}\label{thrm:MAP3}
Let $({\cal W},\succ)$ be a totally ordered preferential structure and $\vapprox_{MAP}$ be a maximum a posteriori entailment with respect to $({\cal W},\succ)$. If there is a model of $\Delta$ then $\Delta\vsim_{({\cal W},\succ)}\alpha$ if and only if $\Delta\vapprox_{MAP}\alpha$.
\end{theorem}
\begin{proof}
Same as Theorem \ref{thrm:MAP1}. The only difference is that such model $w^{*}$ exists uniquely.
\end{proof}
\begin{theorem}\label{thrm:MAP4}
Let $({\cal W},\succ)$ be a totally ordered preferential structure and $\vapprox_{MAP}$ be a maximum a posteriori entailment with respect to $({\cal W},\succ)$. If there is no model of $\Delta$ then $\Delta\vapprox_{MAP}\alpha$ implies $\Delta\vsim_{({\cal W},\succ)}\alpha$, but not vice versa.
\end{theorem}
\begin{proof}
Same as Theorem \ref{thrm:MAP2}.
\end{proof}
\begin{example}
Suppose preferential structure $(\{w_{1}$, $w_{2}$, $w_{3}$, $w_{4}\}$, $\{(w_{1}$, $w_{2})$, $(w_{1}$, $w_{3})$, $(w_{1}$, $w_{4})$, $(w_{3}$, $w_{2})$, $(w_{4}$, $w_{2})\}$ depicted on the left hand side in Figure \ref{fig:tab}. On the right hand side, you can see the probability distribution over valuation functions that preserves the preference order.
\par
Now, $\{a\lor\lnot b\}\vsim_{({\cal W},\succ)}\lnot b$ holds because $\lnot b$ is true in $w_{1}$, which is the $\succ$-maximal model of $\{a\lor\lnot b\}$. Meanwhile, $\{a\lor\lnot b\}\vapprox_{MAP}\lnot b$ holds because $w_{1}\in \argmax_{w}p(w|a\lor\lnot b)$ and $w_{1}\in\llbracket\lnot b\rrbracket$.
\par
In contrast, $\{a\}\not\vsim_{({\cal W},\succ)}\lnot b$ holds because $\lnot b$ is false in $w_{4}$, which is a $\succ$-maximal model of $a$. However, $\{a\}\vapprox_{MAP}\lnot b$ holds because $w_{3}=\argmax_{w}p(w|a)$ and $w_{3}\in\llbracket\lnot b\rrbracket$.
\end{example}
\begin{figure}[t]
\begin{center}
 \includegraphics[scale=0.4]{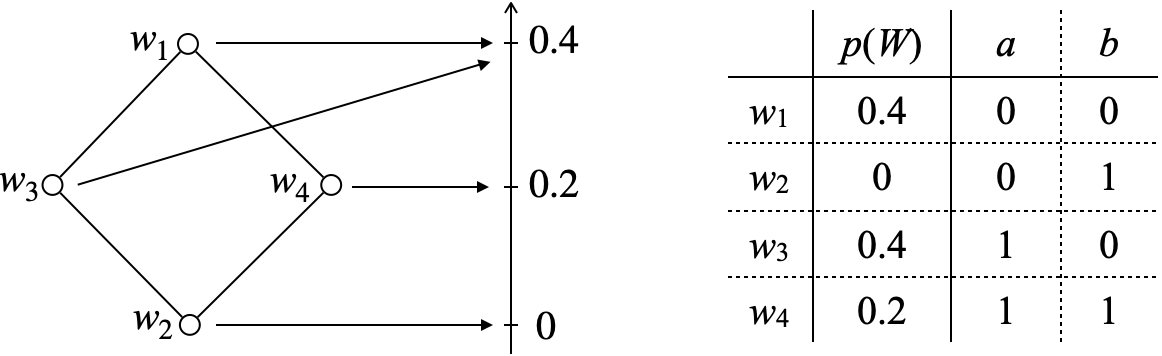}
  \caption{The left graph shows the map from the preferential structure to the prior distribution over valuation functions given in the right table.}
  \label{fig:tab}
\end{center}
\end{figure}

\subsection{Predictive Accuracy}
In this section, we specialise the logical model so that the Bayesian entailment can deal with classification tasks. Correctness of the specialisation is empirically discussed in terms of machine learning using the Titanic dataset available in Kaggle \cite{titanic}, which is an online community of machine learning practitioners. The dataset is used in a Kaggle competition aimed to predict what sorts of people were likely to survive in the Titanic disaster in 1912. Each of 891 data in the dataset contains nine attributes (i.e. ticket class, sex, age, the number of spouses aboard, the number of children aboard, ticket number, passenger fare, cabin number and port of embarkation) and one goal (i.e. survival). In contrast to Table \ref{ex.Int}, the attributes of the Titanic dataset are not generally Boolean variables. We thus treat each attribute with a certain value as a Boolean variable. For example, for the ticket class attribute (abbreviated to $TC$), we assume three Boolean variables $TC=1$, $TC=2$ and $TC=3$, meaning the 1st, 2nd and 3rd class, respectively. In this way, we replace each value of all categorical data with a distinct integer value for identification purpose.
\par
Mathematically, let $D$ be a set of tuples $(\Delta,\alpha)$ where $\Delta$ is a set of formulas and $\alpha$ is a formula. We call $D$ a dataset, $(\Delta,\alpha)$ data, $\Delta$ attributes, and $\alpha$ a goal. The dataset is randomly split into three disjoint sets: 60\% training set, 20\% cross validation set and 20\% test set, denoted by $D_{\text{train}}$, $D_{\text{cv}}$ and $D_{\text{test}}$, respectively. 
\par
We consider the logical model with parameter $(\phi_{1},\phi_{2},...,\phi_{N})$ given by a MLE (maximum likelihood estimate) using the training set and parameter $\mu$ given by a model selection using the cross validation set. Concretely, the MLE is calculated as follows.
%
\begin{eqnarray*}
(\hat{\phi}_{1},...,\hat{\phi}_{N})&\in&\argmax_{\phi_{1},...,\phi_{N}}p(D_{\text{train}}|\phi_{1},...,\phi_{N})\\
&=&\argmax_{\phi_{1},...,\phi_{N}}\prod_{(\Delta,\alpha)\in D_{\text{train}}}p(\alpha,\Delta|\phi_{1},...,\phi_{N})\\
&=&\argmax_{\phi_{1},...,\phi_{N}}\prod_{(\Delta,\alpha)\in D_{\text{train}}}p(\alpha|\phi_{1},...,\phi_{N})\prod_{\beta\in\Delta}p(\beta|\phi_{1},...,\phi_{N})
\end{eqnarray*}
%
In practice, we regard each data in the training set as a possible world, and directly use the training set as a uniform distribution of possibly duplicated possible worlds. This technique results in the same Bayesian predictive entailment although it reduces the cost of the MLE calculation.
\par
Given $(\hat{\phi}_{1},\hat{\phi}_{2},...,\hat{\phi}_{N})$, the model selection is calculated as follows.
\begin{eqnarray*}
\hat{\mu}=\argmax_{\mu}\sum_{(\Delta,\alpha)\in D_{\text{cv}}}\llbracket\Delta\vapprox_{0.5}\alpha\rrbracket,
\end{eqnarray*}
where $\llbracket\Delta\vapprox_{0.5}\alpha\rrbracket=1$ if $\Delta\vapprox_{0.5}\alpha$ holds and $\llbracket\Delta\vapprox_{0.5}\alpha\rrbracket=0$ otherwise. We call the Bayesian entailment defined on the logical model the Bayesian predictive entailment.
\par
We investigate learning performance of the Bayesian predictive entailment in terms of whether or to what extent $\Delta\vapprox_{\theta}\alpha$ holds, for all $(\Delta,\alpha)\in D_{\text{test}}$. Several representative classifiers are compared in Table \ref{tab:comp} in terms of accuracy, AUC (i.e. area under the ROC curve) and the runtime associated with one test datum prediction.
\par
The experimental results were calculated using a MacBook (Retina, 12-inch, 2017) with 1.4 GHz Dual-Core Intel Core i7 processor and 16GB 1867 MHz LPDDR3 memory. We assumed $\theta=0.5$ for the accuracy scores and $\theta\in[0,1]$ for the AUC scores. The best parameter $\mu$ of the Bayesian predictive entailment was selected from $\{0, 0.2, 0.4, 0.6, 0.8, 1\}$. The best number of trees in the forest of the random forest classifier was selected from $\{25, 50, 75, 100, 125, 150\}$. The best additive smoothing parameter of the categorical naive Bayes classifier was selected from $\{0, 0.2, 0.4, 0.6, 0.8, 1\}$. The best number of neighbours of the K-nearest neighbours classifier was selected from $\{5, 10, 15, 20, 25, 30\}$. The best regularisation parameter of the support vector machine classifier was selected from $\{0.001, 0.01, 0.1, 1, 10, 100\}$. All of the remaining parameters were set to be defaults given in scikit-learn 0.23.2.
\par
%
\begin{table}[t]
\caption{Learning performance averaged over one-hundred random splits of training, cross validation and test sets.}
\label{tab:comp}
\begin{center}
\begin{tabular}{c|ccc}
Classifier & Accuracy (std. dev.) & AUC (std. dev.) & Runtime (sec.) \\\hline
Bayesian entailment & 0.785 (0.034)  & \textbf{0.857} (0.032) & \textbf{0.004} \\
Random forest & \textbf{0.790} (0.032) & 0.844 (0.029) & 0.092 \\
Naive Bayes & 0.707 (0.037) & 0.826 (0.034) & 0.009 \\
K-nearest neighbours & 0.718 (0.034) & 0.676 (0.043) & 0.005 \\
Support vector machine & 0.696 (0.035) & 0.652 (0.041) & 0.085
\end{tabular}
\end{center}
\end{table}
%
\section{Discussion and Conclusions}\label{sec:discussion}
There are a number of attempts to combine logic and probability theory, e.g., \cite{adams:98,Fraassen:81b,Fraassen:83,Morgan:83a,Cross:93,Leblanc:79,Leblanc:83b,Pearl:91,Goosens:79,Richardson:06,matthias:13}. They are commonly interested in the notion of probability preservation, rather than truth preservation, where the uncertainty of the conclusion preserves the uncertainty of the premises. They all presuppose and extend the classical entailment. In contrast, this paper gives an alternative entailment without presupposing it.
\par
Besides the preferential entailment, various other semantics for non-monotonic consequence relations have been proposed such as plausibility structure \cite{friedman:96}, possibility structure \cite{dubois:90,Benferhat:03}, ranking structure \cite{goldszmidt:92} and $\varepsilon$-semantics \cite{adams:75,pearl:89}. The common idea of the first three approaches is that $\Delta$ entails $\alpha$ if $\llbracket\Delta\land\alpha\rrbracket\succeq\llbracket\Delta\land\lnot\alpha\rrbracket$ holds given preference relation $\succeq$. However, as discussed in \cite{brewka:97}, it is still unclear how to encode preferences among abnormalities or defaults. A benefit of our approach is that the preferences can be  encoded via Bayesian updating, where the distribution over possible worlds is dynamically updated within probabilistic inference in accordance with observations. Meanwhile, the idea of $\varepsilon$-semantics is that $\Delta$ entails $\alpha$ if $p(\alpha|\Delta)$ is close to one, given a probabilistic knowledge base quantifying the strength of the causal relation or dependency between sentences. They are fundamentally different from our work as we probabilistically model the interaction between models and sentences. The same holds true in the approaches \cite{adams:75,pearl:89,Hawthorne:07a,Hawthorne:07b}.
\par
Naive Bayes classifiers and Bayesian network classifiers work well under the assumption that all or some attributes in data are conditionally independent given another attribute. However, it is rare in practice that the assumption holds in real data. In contrast to the classifiers, our logical model does not need the conditional independence assumption. This is because the logical model always evaluates dependency between possible worlds and attributes, but not dependency among attributes.
\par
In this paper, we introduced a generative model of logical entailment. It formalised the process of how the truth value of a formula is probabilistically generated from the probability distribution over possible worlds. We discussed that it resulted in a simple inference principle that was correct in terms of classical logic, paraconsistent logic, nonmonotonic logic and machine learning. It allowed us to have a general answer to the questions such as how to logically infer from inconsistent knowledge, how to rationally handle defeasibility of everyday reasoning, and how to probabilistically infer from noisy data without a conditional dependence assumption.

\section*{Appendix}
\begin{proof}[Proposition 1]
We abbreviate $W=w$ to $w$ for simplicity. Since $\llbracket\alpha=0\rrbracket_{w}=1-\llbracket\alpha=1\rrbracket_{w}$, we have
\begin{eqnarray*}
p(\alpha=0|w)+p(\alpha=1|w)&=&\mu^{\llbracket\alpha=0\rrbracket_{w}}(1-\mu)^{1-\llbracket\alpha=0\rrbracket_{w}}+\mu^{\llbracket\alpha=1\rrbracket_{w}}(1-\mu)^{1-\llbracket\alpha=1\rrbracket_{w}}\\
&=&\mu^{1-\llbracket\alpha=1\rrbracket_{w}}(1-\mu)^{\llbracket\alpha=1\rrbracket_{w}}+\mu^{\llbracket\alpha=1\rrbracket_{w}}(1-\mu)^{1-\llbracket\alpha=1\rrbracket_{w}}.
\end{eqnarray*}
(1) holds because both $p(\alpha|w)$ and $p(w)$ cannot be negative. If $\llbracket\alpha=1\rrbracket_{w}=1$ then $p(\alpha=0|w)+p(\alpha=1|w)=(1-\mu)+\mu=1$. If $\llbracket\alpha=1\rrbracket_{w}=0$ then $p(\alpha=0|w)+p(\alpha=1|w)=\mu+(1-\mu)=1$. Now, (2) is shown as follows.
\begin{eqnarray*}
p(\alpha=0)+p(\alpha=1)&=&\sum_{w}p(\alpha=0|w)p(w)+\sum_{w}p(\alpha=1|w)p(w)\\
&=&\sum_{w}p(w)\{p(\alpha=0|w)+p(\alpha=1|w)\}\\
&=&\sum_{w}p(w)=1
\end{eqnarray*}
(3) is shown as follows. From (2), it is sufficient to show only case $i=1$ because case $i=0$ can be developed as follows.
\begin{eqnarray*}
1-p(\alpha\lor\beta=1)&=&1-\{p(\alpha=1)+p(\beta=1)-p(\alpha\land\beta=1)\}.
\end{eqnarray*}
Now, it is sufficient to show $p(\alpha\lor\beta=1|w)=p(\alpha=1|w)+p(\beta=1|w)-p(\alpha\land\beta=1|w)$ since case $i=1$ can be developed as follows.
\begin{eqnarray*}
\sum_{w}p(\alpha\lor\beta=1|w)p(w)=\sum_{w}\{p(\alpha=1|w)+p(\beta=1|w)-p(\alpha\land\beta=1|w)\}p(w)
\end{eqnarray*}
By case analysis, the right expression is shown to have
\begin{eqnarray}
(1-\mu)+(1-\mu)-(1-\mu)&=&1-\mu\label{1}\\
(1-\mu)+\mu-(1-\mu)&=&\mu \label{2}\\
\mu +(1-\mu)-(1-\mu)&=&\mu\label{3}\\
\mu+\mu-\mu&=&\mu \label{4}
\end{eqnarray}
where (\ref{1}), (\ref{2}), (\ref{3}) and (\ref{4}) are obtained in the cases ($\llbracket\alpha=1\rrbracket_{w}=\llbracket\beta=1\rrbracket_{w}=0$), ($\llbracket\alpha=1\rrbracket_{w}=0$ and $\llbracket\beta=1\rrbracket_{w}=1$), ($\llbracket\alpha=1\rrbracket_{w}=1$ and $w\in\llbracket\beta=1\rrbracket_{w}=0$), and ($\llbracket\alpha=1\rrbracket_{w}=\llbracket\beta=1\rrbracket_{w}=1$), respectively.  All of the results are consistent with the left expression, i.e., $p(\alpha\lor\beta=1|w)$.
\end{proof}
\begin{proof}[Proposition 2]
For all $w$, $p(\alpha=1|w)=\mu$ if and only if $p(\lnot\alpha=1|w)=1-\mu$, and $p(\alpha=1|w)=1-\mu$ if and only if $p(\lnot\alpha=1|w)=\mu$. Therefore, $p(\alpha=1|w)=1-p(\lnot\alpha=1|w)$. From (2) of Proposition \ref{kolmogorov}, we have 
\begin{eqnarray*}
p(\alpha=1)&=&\sum_{w}p(\alpha=1|w)p(w)\\
&=&\sum_{w}\{1-p(\lnot\alpha=1|w)\}p(w)\\
&=&\sum_{w}p(\lnot\alpha=0|w)p(w)=p(\lnot\alpha=0).
\end{eqnarray*}
\end{proof}

\bibliographystyle{elsarticle-num}
\bibliography{btxkido}
\end{sloppypar}
\end{document}